\documentclass[english,a4paper,11pt]{article}
\usepackage[english]{babel}
\usepackage{amsmath, amsxtra, amsfonts, amssymb, amstext}
\usepackage{amsthm}
\usepackage{booktabs}
\usepackage{fullpage}
\usepackage{nicefrac}
\usepackage{xspace}
\usepackage[noadjust]{cite}
\usepackage{url}
\usepackage{graphics,color}
\usepackage[colorlinks]{hyperref}
\definecolor{linkblue}{rgb}{0.1,0.1,0.8}
\hypersetup{colorlinks=true,linkcolor=linkblue,filecolor=linkblue,urlcolor=linkblue,citecolor=linkblue}
\usepackage[algo2e,ruled,vlined,linesnumbered]{algorithm2e}
\SetKwFor{Function}{function}{is}{end function}
\SetKwFor{Subroutine}{subroutine}{is}{end subroutine}
\newcommand{\assign}{\leftarrow}
\SetKw{Input}{input}
\SetKw{Output}{output}
\SetKw{Initialization}{initialization}
\newcommand{\set}[2]{\{#1 \; | \; #2\}}

\newtheorem{theorem}{Theorem}
\newtheorem{lemma}[theorem]{Lemma}

\newtheorem{corollary}[theorem]{Corollary}
\newtheorem{definition}[theorem]{Definition}

\newtheorem{remark}[theorem]{Remark}

\newcommand{\comment}[1]{\textcolor{red}{#1}}

\allowdisplaybreaks[1]

\renewcommand{\epsilon}{\varepsilon}
\newcommand{\eps}{\varepsilon}

\newcommand{\F}{\mathcal{F}}

\newcommand{\onemax}{\textsc{OneMax}\xspace}
\newcommand{\jump}[1]{\textsc{Jump}_{#1}\xspace}

\newcommand{\Jump}{\textsc{Jump}\xspace}

\newcommand{\UBB}[1]{\mathrm{UBB}_{#1}}

\newcommand{\uniformSample}{{\tt uniform}\xspace}

\newcommand{\flip}[1]{{\tt flip}$_{#1}$\xspace}
\newcommand{\flipB}{{\tt flip}\xspace}
\renewcommand{\complement}{{\tt complement}\xspace}
\newcommand{\copySecondIntoFirstWhereDifferentFromThird}{{\tt copySecondIntoFirstWhereDifferentFromThird}\xspace}
\newcommand{\selectBit}{{\tt selectBits}\xspace}
\newcommand{\mix}{{\tt mix}\xspace}

\newcommand{\movefirst}{{\tt movefirst}\xspace}
\newcommand{\flipwheredifferent}{{\tt flipWhereDifferent}\xspace}
\newcommand{\flipwhereequal}{{\tt flipWhereEqual}\xspace}

\newcommand{\estimate}{{\tt estimate}\xspace}
\newcommand{\simulateOnSubcube}{{\tt simulateOnSubcube}\xspace}
\newcommand{\breakandgotoone}{{\tt breakandgoto1}\xspace}

\DeclareMathOperator{\sgn}{sgn}

\begin{document}

\title{Unbiased Black-Box Complexities of Jump Functions\thanks{This paper is based on results published in the conference versions~\cite{DoerrKW11,DoerrDK14}.}}

\author{Benjamin Doerr$^{1}$\and
Carola Doerr$^{2,3}$\and
Timo K\"otzing$^4$}
\date{
$^1$\'Ecole Polytechnique, Palaiseau, France\\
$^2$Sorbonne Universit\'es, UPMC Univ Paris 06,  UMR 7906, LIP6, 75005 Paris, France\\
$^3$CNRS, UMR 7906, LIP6, 75005 Paris, France\\
$^4$Friedrich-Schiller-Universit\"at Jena, Germany\\[2ex]
\today
}

\maketitle

{\sloppy
\begin{abstract}
We analyze the unbiased black-box complexity of jump functions with small, medium, and large sizes of the fitness plateau surrounding the optimal solution.  

Among other results, we show that when the jump size is $(1/2 - \varepsilon)n$, that is, only a small constant fraction of the fitness values is visible, then the unbiased black-box complexities for arities $3$ and higher are of the same order as those for the simple \textsc{OneMax} function. Even for the extreme jump function, in which all but the two fitness values $n/2$ and $n$ are blanked out, polynomial-time mutation-based (i.e., unary unbiased) black-box optimization algorithms exist. This is quite surprising given that for the extreme jump function almost the whole search space (all but a $\Theta(n^{-1/2})$ fraction) is a plateau of constant fitness.  

To prove these results, we introduce new tools for the analysis of unbiased black-box complexities, for example, selecting the new parent individual not by comparing the fitnesses of the competing search points, but also by taking into account the (empirical) expected fitnesses of their offspring. 
\end{abstract} 


\section{Introduction}%
\label{sec:introduction}

The analysis of black-box complexities in evolutionary computation aims in several complementary ways at supporting the development of superior evolutionary algorithms. By comparing the run time of currently used randomized search heuristics (RSHs) with the one of an optimal black-box algorithm, it allows a fair evaluation of how good today's RSH are. With specific black-box complexity notions, we can understand how algorithm components and parameters such as the population size, the selection rules, or the sampling procedures influence the run time of RSH. Finally, research in black-box complexity proved to be a source of inspiration for developing new algorithmic ideas that lead to the design of better search heuristics.

In this work, we analyze the unbiased black-box complexity of jump functions, which are observed as difficult for evolutionary approaches because of their large plateaus of constant (and low) fitness. 
Our results show that, surprisingly, even extreme jump functions that reveal only the three different fitness values $0$, $n/2$, and $n$ can be optimized by a mutation-based unbiased black-box algorithm in polynomial time. 
We introduce new methods that facilitate our analyses. The perhaps most interesting one is a routine that creates a number of samples from which it estimates the distance of the current search point to the fitness layer $n/2$. Our algorithm thus benefits significantly from analyzing some non-standard statistics of the fitness landscape. We believe this to be an interesting idea that should be investigated further. Our hope is that it can be used to design new search heuristics.

\subsection{Black-Box Complexity}

Black-box complexity studies how many function evaluations are needed in expectation by an optimal black-box algorithm until it queries for the first time an optimal solution for the problem at hand. 
Randomized search heuristics, like evolutionary algorithms, simulated annealing, and ant colony algorithms, are typical black-box optimizers: they are typically problem-independent and as such they learn about the problem to be solved only by generating and evaluating search points. 
The black-box complexity of a problem is thus a lower bound for the number of fitness evaluations needed by any search heuristic to solve it.

Several black-box complexity notions covering different aspects of randomized search heuristics exist, for example the unrestricted model~\cite{DrosteJW06}, which does not restrict in any way the sampling or selection procedure of the algorithm, the ranking-based model~\cite{TeytaudG06, DoerrW11}, in which the algorithms are required to base their selection only on relative and not on absolute fitness values, the memory-restricted model~\cite{DrosteJW06, DoerrW12}, in which the algorithm can store only a limited number of search points and their corresponding fitness values, and the unbiased model~\cite{LehreW12}, which requires the algorithms to treat the representation of the search points symmetrically. 
By comparing the respective black-box complexities of a problem, one learns how the run time of RSHs is influenced by certain algorithmic choices such as the population size, the use of crossover, the selection rules, etc.
 
For all existing black-box models, however, the typical optimal black-box algorithm is a highly problem-tailored algorithm that is not necessarily nature-inspired. 
Still we can learn from such ``artificial'' algorithms about RSH, as has been shown in~\cite{DoerrDE13}. In that work, a new genetic algorithm is presented that optimizes the \onemax function in run time $o(n \log n)$, thus showing that the $o(n \log n)$ bound for the $2$-ary unbiased black-box complexity of $\onemax$ found in~\cite{DoerrJKLWW11} is not as unnatural as it might have seemed at first. 

Here in this work we consider unbiased black-box complexities. The unbiased black-box model is one of the standard models for analyzing the influence of the arity on the performance of optimal black-box algorithms. 
It was originally defined by Lehre and Witt~\cite{LehreW12} for bit string representations and has later been generalized to domains different from bit strings~\cite{ABB, DoerrKLW13}. 
For bit string representations, the unbiased model requires the optimizing algorithms to treat different positions of the bit strings equally, similarly with the two different possible bit contents (thus the term ``unbiased''). For example, unbiased algorithms are not allowed to explicitly write a $1$ or a $0$ at a specific position of a bit string to be evaluated; instead, such algorithms can either sample a bit string uniformly at random, or generate one from previously evaluated solutions via operators which are unbiased (i.e., treat positions and bit contents equally). See Section~\ref{sec:model} for a detailed description of the model.

\subsection{Jump Functions}

In this paper we are concerned with the optimization of functions mapping bit strings of fixed length (i.e., elements of the hypercube $\{0,1\}^n$) to real numbers; such functions are called \emph{pseudo-Boolean}. A famous pseudo-Boolean function often considered as a test function for optimization is the \onemax function, mapping any $x \in \{0,1\}^n$ to the number of $1$s in $x$ (the \emph{Hamming weight} of $x$).

Other popular test functions are the jump functions. For a non-negative integer $\ell$, we define the $\jump{\ell}$ as derived from \onemax by ``blanking out'' any useful information within the strict $\ell$-neighborhood of the optimum (and the minimum) by giving all these search points a fitness value of $0$. 
In other words, $\jump{\ell}(x) = \onemax(x)$ if $\onemax(x) \in \{0\} \cup \{\ell+1,\ldots,n-\ell-1\} \cup \{n\}$ and $\jump{\ell}(x) = 0$ otherwise. This definition is mostly similar to the two, also not fully agreeing, definitions used in~\cite{DrosteJW02} and~\cite{LehreW10}.

$\Jump$ functions are well-known test functions for randomized search heuristics. Droste, Jansen, and Wegener~\cite{DrosteJW02} analyzed the optimization time of the (1+1) evolutionary algorithm on $\Jump$ functions. From their work, it is easy to see that for our definition of $\Jump$ functions, a run time of $\Theta(n^{\ell+1})$ for the (1+1) evolutionary algorithm on $\jump{\ell}$ follows for all $\ell \in \{1, \ldots, \lfloor n/2 \rfloor-1\}$. We are not aware of any natural mutation-based randomized search heuristic with significantly better performance (except for large $\ell$, where simple random search with its $\Theta(2^n)$ run time becomes superior). For all $\ell$, Jansen and Wegener~\cite{JansenW02} present a crossover-based algorithm for $\jump{\ell}$. With an optimal choice for the parameter involved, which in particular implies a very small crossover rate of $O(1/n)$, it has an optimization time of $O(n \log^3 n)$ for constant $\ell$ (this is mistakenly stated as $O(n \log^2 n \log\log n)$ on the last line of p.~60 of the paper, but it is clear from the proof that this is only a typo) and an optimization time of $O(n^{2c + 1} \log n)$ for $\ell = \lceil c \log n \rceil$, $c$ a constant.

\subsection{Our Results}

\begin{table*}[t]
\label{tab:results}
\begin{center}
\begin{tabular}{c|c|c|c}
& Short Jump & 
Long Jump & 
Extreme Jump\\
\rule{5mm}{0cm}Arity\rule{5mm}{0cm} & 
$\ell = O(n^{1/2 - \eps})$\rule{1mm}{0cm} & 
$\ell = (1/2 - \varepsilon)n$\rule{1mm}{0cm} & 
$\ell = n/2 - 1$\rule{1mm}{0cm}\\ 
\hline
$k = 1$ &  $\Theta(n \log n)$ 
& $O(n^2)$  & $O(n^{9/2})$ \\
$k = 2$ & $O(n)$  
& $O(n \log n)$  & $O(n \log n)$ \\
$3 \leq k \leq \log n$ & $O(n / k)$ & $O(n / k)$ & $\Theta(n)$ 
\end{tabular}	
\end{center}
\caption{Unbiased black-box complexities of $\jump{\ell}$ for different regimes of $\ell$.}
\end{table*}

We analyze the unbiased black-box complexity of $\Jump$ functions for a broad range of jump sizes $\ell$. We distinguish between \emph{short}, \emph{long}, and \emph{extreme} jump functions for $\ell = O(n^{1/2 - \eps})$, $ell=(1/2 - \varepsilon)n$, and $\ell=n/2-1$, respectively. Our findings are summarized in Table~\ref{tab:results}.

Contrasting the runtime results for classic evolutionary approaches on $\Jump$, we show that for jump functions with \emph{small} jump sizes $\ell = O(n^{1/2 - \eps})$, the $k$-ary unbiased black-box complexities are of the same order as those of the easy $\onemax$ test function (which can be seen as a $\Jump$ function with parameter $\ell = 0$). As an intermediate result we prove (Lemma~\ref{lem:onemaxSimulation}) that a black-box algorithm having access to a jump function with $\ell= O(n^{1/2 - \eps})$ can retrieve (with high probability) the true $\onemax$ value of a search point using only a constant number of queries. This implies that we get the same run time bounds for short jump functions as are known for \onemax. For $k=1$ this is $\Theta(n \log n)$~\cite{LehreW12}, for $k=2$ its is $O(n)$~\cite{DoerrJKLWW11}, and for $3 \leq k \leq \log n$ this is $O(n/k)$~\cite{DoerrW12g}.

A result like Lemma~\ref{lem:onemaxSimulation} is not to be expected to hold for larger values of $\ell$. Nevertheless, we show that also long jump functions, where $\ell$ can be as large as $(1/2 - \varepsilon)n$, have unbiased black-box complexities of the same asymptotic order as $\onemax$ for arities $k \ge 3$. For $k=2$ we get a bound of $O(n \log n)$ and for $k=1$ we get $O(n^2)$, both surprisingly low black-box complexities. 
Even for the case of extreme jump functions, where $\ell = n/2 -1$ and $n$ even (a jump function revealing only the optimum and the fitness $n/2$), we are able to show polynomial unbiased black-box complexities for all arities $k \ge 1$. 

Note that already for long jump functions, the fitness plateau that the algorithms have to cross has exponential size. For the extreme jump function, even all but a $\Theta(n^{-1/2})$ fraction of the search points form one single fitness plateau.
This is the reason why none of the popular randomized search algorithms will find the optimum of long and extreme jump functions in subexponential time.

Our results indicate that even without the fitness function revealing useful information for search points close the optimum, efficient optimization is still possible in the framework of unbiased black-box algorithms. 

The results regarding short jump can be found in Section~\ref{sec:shortJump}, the results on long jump in Section~\ref{sec:longJump2}, while Section~\ref{sec:extremeJump} gives the details on extreme jump. Note that the bound on the binary unbiased black-box complexity of long jump follows from the same bound on extreme jump. The lower bounds partially follow from a more general result of independent interest (Theorem~\ref{thm:lowerBoundForMapped}), which implies that for all $0 \le \ell_1 \le \ell_2$, for all $k$ the $k$-ary unbiased black-box complexity of $\jump{\ell_1}$ is less than or equal to the one of $\jump{\ell_2}$. 

\subsection{Methods}

In order to show the upper bounds on the black-box complexities we give efficient algorithms optimizing the different jump functions. For arity $k=1$, these algorithms are based on iteratively getting closer to the optimum; however we do not (and in fact cannot) rely on fitness information about these closer search points: the fitness is $0$ in almost all cases. Instead, we rely on the \emph{empirical expected fitness} of offspring. For this we use mutation operators that have a good chance of sampling offspring with non-zero fitness. We show that already a polynomial number of samples suffices to distinguish search points whose fitness differs only minimally. In order to minimize the number of samples required, we choose this number \emph{adaptively} depending on the estimated number of $1$s in the search point to be evaluated; we also allow fairly frequent incorrect decisions, as long as the overall progress to the optimum is guaranteed.

In one of our proofs we make use of an additive Chernoff bound for negatively correlated variables. This bound is implicit in a paper by Panconesi and Srivnivasan~\cite{PanconesiS97} and is of independent interest.

\section{The Unbiased Black-Box Model}
\label{sec:model}

The unbiased black-box model introduced in~\cite{LehreW12} is by now one of the standard complexity models in evolutionary computation. In particular the unary unbiased model gives a more realistic complexity estimate for a number of functions than the original unrestricted black-box model of Droste, Jansen, and Wegener~\cite{DrosteJW06}.
An important advantage of the unbiased model is that it allows us to analyze the influence of the arity of the sampling operators in use. 
In addition, new search points can be sampled only either uniformly at random or from distributions that depend on previously generated search points in an \emph{unbiased} way. 
In this section we briefly give a brief definition of the unbiased black-box model, pointing the interested reader to~\cite{LehreW12} and~\cite{DoerrW12g} for a more detailed introduction. 

For all non-negative integers $k$,  
a \emph{$k$-ary unbiased distribution} $\big(D(\cdot \,|\, y^{(1)},\ldots,y^{(k)})\big)_{y^{(1)},\ldots,y^{(k)} \in \{0,1\}^n}$ is a family of probability distributions over 
$\{0,1\}^n$ such that for all \emph{inputs} $y^{(1)},\ldots,y^{(k)} \in \{ 0,1\}^n$ the following two conditions hold.
\begin{enumerate}
	\item~[$\oplus$-invariance] $\forall x,z \in \{0,1\}^n: D(x \mid y^{(1)},\ldots,y^{(k)}) = D(x \oplus z \mid y^{(1)} \oplus z,\ldots,y^{(k)}\oplus z);$
	\item~ [permutation-invariance] $\forall x \in \{0,1\}^n \, \forall \sigma\in S_n:$\\
	$\quad D(x \mid y^{(1)},\ldots,y^{(k)}) = D(\sigma(x) \mid \sigma(y^{(1)}), \ldots, \sigma(y^{(k)}))$,
\end{enumerate}
where $\oplus$ is the bitwise exclusive-OR, $S_n$ the set of all permutations of the set $[n]:=\{1,2,\ldots,n\}$, and $\sigma(x):=x_{\sigma(1)}\cdots x_{\sigma(n)}$ for $x=x_1 \cdots x_n \in \{ 0,1\}^n$.
 
An operator sampling from a $k$-ary unbiased distribution is called a \emph{$k$-ary unbiased variation operator}. 

\begin{algorithm2e}[t]
	\textbf{Initialization:} Sample $x^{(0)} \in \{0,1\}^n$ uniformly at random and query $f(x^{(0)}).$\\
	 \textbf{Optimization:}
\For{$t=1,2,3,\ldots$ \textbf{\upshape{until}} termination condition met}{
		Depending on $\big(f(x^{(0)}),\ldots, f(x^{(t-1)})\big)$ choose up to $k$ indices $i_1,\ldots,i_k \in [0..t-1]$ and a $k$-ary unbiased distribution $D(\cdot \mid x^{(i_1)},\ldots,x^{(i_k)})$. \\
		Sample $x^{(t)}$ according to $D(\cdot \mid x^{(i_1)},\ldots,x^{(i_k)})$ and query $f(x^{(t)})$.
}
\caption{Scheme of a $k$-ary unbiased black-box algorithm}
\label{alg:unbiasedAlgo}
\end{algorithm2e} 

A $k$-ary unbiased black-box algorithm is one that follows the scheme of Algorithm \ref{alg:unbiasedAlgo} (here and in the following with $[0..k]$ we abbreviate $[k] \cup \{0\}$). 
The \emph{$k$-ary unbiased black-box complexity}, denoted $\UBB{k}(\F)$, of some class of functions $\F$ is the minimum complexity of $\F$ with respect to all $k$-ary unbiased black-box algorithms, where, naturally, the complexity of an algorithm $A$ for $\F$ is the maximum expected number of black-box queries that $A$ performs on a function $f \in \F$ until it queries for the first time a search point of maximal fitness.
We let \emph{$\ast$-ary unbiased black-box complexity} be based on the model in which operators of arbitrary arity are allowed.

The unbiased black-box model includes most of the commonly studied search heuristics, such as many $(\mu + \lambda)$ and $(\mu,\lambda)$ evolutionary algorithms (EAs), Simulated Annealing, the Metropolis algorithm, and the Randomized Local Search algorithm. 

We recall a simple remark from~\cite{DoerrDK14ArtInt} which helps us shortening some of the proofs in the subsequent sections. 
\begin{remark}
\label{rem:highprobability}\label{REM:HIGHPROB}
Suppose for a problem $P$ there exists a black-box algorithm $A$ that, with constant success probability, solves $P$ in $s$ iterations (that is, queries an optimal solution within $s$ queries). Then the black-box complexity of $P$ is at most $O(s)$.
\end{remark}

A useful tool for proving lower bounds is Theorem~\ref{thm:lowerBoundForMapped}. It formalizes the intuition that the black-box complexity of a function can only get harder if we ``blank out'' some of the fitness values. This is exactly the situation of the \Jump functions, whose definition we repeat here for the sake of completeness.

For all $\ell < n/2$, $\jump{\ell}$ is the function that assigns to each $x \in \{0,1\}^n$ fitness
$$
\jump{\ell}(x) = 
\begin{cases}
n,				&\mbox{if }|x|_1 = n;\\
|x|_1,		&\mbox{if }\ell < |x|_1 < n-\ell;\\
0,				&\mbox{otherwise,}
\end{cases}
$$
where $|x|_1:=\onemax(x):=\sum_{i=1}^n{x_i}$ denotes the number of $1$s in $x$ (also known as the \emph{Hamming-weight of $x$}).

\begin{theorem}
\label{thm:lowerBoundForMapped} 
For all sets of pseudo-Boolean functions $C$, all $k \in \mathbb{N}$, and all $f: \mathbb{R} \rightarrow \mathbb{R}$ such that $\forall g \in C: \{x \mid f(g(x)) \mbox{ optimal}\,\} \subseteq \{x \mid g(x) \mbox{ optimal}\,\}$, we have $\UBB{k}(C) \leq \UBB{k}(f(C))$.
\end{theorem}

\begin{proof}
Let $C$, $k$, and $f$ be as in the statement of the theorem. Let $A$ be any $k$-ary unbiased black-box algorithm for $f(C)$. We derive from this a $k$-ary unbiased black-box algorithm for $C$ by using queries to $g \in C$ and then mapping the resulting objective value with $f$. Clearly, $A'$ finds an optimum of $g \in C$ after no more expected queries than $A$ for $f \circ g$, using the condition on the set of optimal points. Thus, the theorem follows.
\end{proof}

From Theorem~\ref{thm:lowerBoundForMapped} we immediately obtain a lower bound of $\Omega(n / \log n )$ for the unbiased black-box complexities of jump functions. The theorem implies that the $k$-ary unbiased black-box complexity of \onemax is a lower bound of that of any jump function. In general, the $k$-ary unbiased black-box complexity of any pseudo-Boolean function $f$ is at least the unrestricted black-box complexity of the class of functions obtained from $f$ by first applying an auto\-morphism of the hypercube $\{0,1\}^n$. That the latter for $\onemax$ is $\Omega(n/\log n)$ was shown independently in~\cite{DrosteJW06} and~\cite{Erd63}. A similar line of arguments will prove the lower bound for extreme jump functions in Section~\ref{sec:extremeJump}.

The lower bound for the unary unbiased black-box complexity of $\Jump$ follows immediately from the $\Omega(n \log n)$ bound proven in~\cite[Theorem 6]{LehreW12} for all pseudo-Boolean functions with unique global optimum.

\section{Short Jump Functions}
\label{sec:shortJump}

The key idea for obtaining the bounds on short jump functions, i.e., jump functions with jump size $\ell \in O(n^{1/2 - \varepsilon})$, is the following lemma. It shows that one can compute, with high probability, the \onemax value of any search point $x$ with few black-box calls to $\jump{k}$. With this, we can orient ourselves on the large plateau surrounding the optimum and thus revert to the problem of optimizing $\onemax$. 

We collect these computations in a subroutine, to be called by black-box algorithms.

\begin{lemma}
\label{lem:onemaxSimulation}\label{LEM:OMSIMU}
For all constants $\epsilon,c>0$ and all $\ell \in O(n^{1/2 - \varepsilon})$, there is a unary unbiased subroutine $s$ using $O(1)$ queries to $\jump{\ell}$ such that, for all bit strings $x$, $s(x) = \onemax(x)$ with probability $1- O(n^{-c})$.
\end{lemma}

\begin{proof} We assume $n$ to be large enough so that $\ell \leq n/4$.
We use a unary unbiased variation operator \flip{\ell}, which samples uniformly an $\ell$-neighbor (a bit string which differs in exactly $\ell$ positions) of the argument. Next we give the subroutine~$s$, which uses $\jump{\ell}$ to approximate \onemax as desired, see Algorithm~\ref{alg:subroutine}. Intuitively, the subroutine samples $t = \lceil c/(2\varepsilon) \rceil \in O(1)$ bit strings in the $\ell$-neighborhood of $x$; if $|x|_1 \geq n-\ell$ then it is likely that at least once only $1$s of $x$ have been flipped, leading to a $\jump{\ell}$-value of $|x|_1-\ell$; as no sample will have a lower $\jump{\ell}$-value, adding $\ell$ to the minimum non-$0$ fitness of one of the sampled bit strings gives the desired output. The case of $x$ with $|x|_1 \leq \ell$ is analogous.
\begin{algorithm2e}[Hh]
 \Subroutine{$s(x)$}{
 		\lIf{$\jump{\ell}(x) \neq 0$}{\Output $\jump{\ell}(x)$\;}
		$M$ $\assign$ $\set{\jump{\ell}(\mbox{\flip{\ell}}(x))}{m \in [\lceil c/(2\varepsilon) \rceil]}$\;
		\lIf{$\max(M) < n/2$}{$m \assign \max(M) - \ell$\;}
		\lElse{$m \assign \min(M\setminus \{0\}) + \ell$\;}
 		\Output $m$\;
 }
\caption{Simulation of \onemax using the jump function.}
\label{alg:subroutine}
\end{algorithm2e}

Clearly, the subroutine is correct with certainty on all $x$ with $\ell < |x|_1 < n-\ell$. The other two cases are nearly symmetric, so let us only analyze $x$ with $|x|_1 \geq n-\ell$. Obviously, the return value of the subroutine is correct if and only if at least one of the $t$ samples flips only $1$s in $x$ (note that $\max(M) > n/2$ holds due to $\ell \leq n/4$). We denote the probability of this event with~$p$. We start by bounding the probability that a single sample flips only $1$s. We choose which $k$ bits to flip iteratively so that, after $i$ iterations, there are at least $n-\ell-i$ bit positions with a $1$ out of $n-i$ unchosen bit positions left to choose. This gives the bound of 
$$
\begin{array}{l}
\Big(\frac{n-\ell}{n}\Big)\cdot\left(\frac{n-\ell-1}{n-1}\right) \cdots \left(\frac{n-\ell-(\ell-1)}{n-(\ell-1)}\right)
= \prod_{i=0}^{\ell-1}\left(1-\frac{\ell}{n-i}\right) \geq \left(1-\frac{\ell}{n-\ell}\right)^\ell \geq \left(1-\frac{\ell^2}{n-\ell}\right),
\end{array}
$$
using Bernoulli's inequality.
Thus, we have
$$p \geq 1 - \left(\frac{\ell^2}{n-\ell}\right)^t \geq 1 - \left(\frac{4\ell^2}{n}\right)^t \geq 1 - \left(4n^{-2\varepsilon}\right)^t \geq 1 - 4^tn^{-c}.$$
\end{proof}

With Lemma~\ref{lem:onemaxSimulation} at hand, the results stated in Table~\ref{tab:results} follow easily from the respective \onemax bounds proven in~\cite{DoerrJKLWW11, DoerrW12,DrosteJW06}. 
 
 \begin{theorem}
 \label{thm:jumpFast}
For $\epsilon>0$ and $\ell \in O(n^{1/2 - \varepsilon})$, 
the unbiased black-box complexity of $\jump{\ell}$ is  
$O(n \log n)$ for unary variation operators and it is $O(n / k)$ for $k$-ary variation operators with $2 \leq k \leq \log n$.
\end{theorem}

\begin{proof}
First note that the above black-box complexities claimed for $\jump{\ell}$ are shown for \onemax in~\cite{DrosteJW06} for $k=1$, in~\cite{DoerrJKLWW11} for $k=2$, and in~\cite{DoerrW12g} for $3 \leq k \leq \log n$. 

We use Lemma~\ref{lem:onemaxSimulation} with $c=4$ and run the unbiased black-box algorithms of the appropriate arity for \onemax; all sampled bit strings are evaluated using the subroutine $s$. 
Thus, this algorithm samples as if working on \onemax and finds the bit string with all $1$s after the desired number of iterations. 
Note that, for up to $n \log n$ uses of $s$, we expect no more than $n \log n  O(n^{-4}) \leq O(n^{-2})$ incorrect evaluations of $s$. Therefore, there is a small chance of failing, and the claim follows from Remark~\ref{rem:highprobability}.
\end{proof}

\comment{wollen wir das Folgende behalten?}
Note that the subroutine from Lemma~\ref{lem:onemaxSimulation} requires to know the parameter $\ell$; however, this subroutine can be modified to work without that knowledge as follows. 
The first time that the subroutine samples a search point with fitness $0$ it will determine $\ell$; after knowing $\ell$, it will work as before (and before sampling a search point of fitness $0$, it does not need to know). The parameter $\ell$ is determined by sampling sufficiently many $i$-neighbors of the search point with fitness $0$, starting with $i=1$ and stopping when a search point with fitness $\neq 0$ is found. This search point will have maximum fitness among all non-optimal search points, equal to $n-\ell-1$. From this fitness and $n$, the subroutine can infer $\ell$.

There may be cases when one of the algorithms implicit in the proof of Theorem~\ref{thm:jumpFast} never samples a search point with fitness $0$ and does not have to determine $\ell$. In this case, such an algorithm will optimize the target function without completely learning it. However, in a second phase after finding the optimum, an algorithm could determine $\ell$ with a binary search, as $\ell$ equals the largest distance from the optimum at which all and any search point has fitness $0$. This phase requires $O(\log(n))$ queries.

\section{Long Jump Functions}
\label{sec:longJump2}

In this section we give bounds on long jump functions; we start with a bound on the ternary black-box complexity, followed by a bound on the unary black-box complexity. Note that the bound on the binary unbiased black-box complexity of long jump follows from the same bound on extreme jump.

\subsection{Ternary Unbiased Optimization of Long Jump Functions}
\label{sec:ternaryLongJump2}

We show that ternary operators allow for solving the problem independently in different parts of the bit string, and then combining the partial solutions. This has the advantage that, as done in Section~\ref{sec:shortJump}, we can revert to optimizing \onemax, and the missing fitness values will not show in any of the partial problems.

We start with a lemma regarding the possibility of simulating unbiased algorithms for \onemax on subsets of the bits.

\begin{lemma}\label{lem:simulateOnSubcube}
For all bit strings $x,y \in \{0,1\}^n$ we let $[x,y] = \{z \in \{0,1\}^n \mid \forall i\leq n: x_i = y_i \Rightarrow x_i = z_i\}$ (this set is isomorphic to a hypercube). Let $A$ be a $k$-ary unbiased black-box algorithm optimizing \onemax with constant probability in time at most $t(n)$.
Then there is a $(k+2)$-ary unbiased black-box subroutine \simulateOnSubcube as follows. 
\begin{itemize}
	\item Inputs to \simulateOnSubcube are $x,y \in \{0,1\}^n$ and the Hamming distance $a$ of $x$ and $y$; $x$ and $y$ are accessible as search points sampled previous to the call of the subroutine.
	\item \simulateOnSubcube has access to an oracle returning $\onemax(z)$ for all $z \in [x,y]$. 
	\item After at most $t(a)$ queries \simulateOnSubcube has found the $z \in [x,y]$ with maximal \onemax value with constant probability.
\end{itemize}
\end{lemma}

\begin{proof}
Let $x$ and $y$ with Hamming distance $a$ be given as detailed in the statement of the theorem. Note that $[x,y]$ is isomorphic to $\{0,1\}^a$. Without loss of generality, assume that $x$ and $y$ differ on the first $a$ bits, and let $\tilde{x}$ be the last $n-a$ bits of $x$ (which equal the last $n-a$ bits of $y$). Thus, $[x,y] = \{z\tilde{x} \mid z \in \{0,1\}^a\}$.

We employ $A$ optimizing $\{0,1\}^a$. Sampling a uniformly random point $z \tilde{x}$ in $[x,y]$ is clearly unbiased in $x$ and $y$. However, the resulting \onemax value is not the value that $A$ requires, unless $\tilde{x}$ is the all-$0$ string. In order to correct for this, we need to know the number of ones $|\tilde{x}|_1$ in $\tilde{x}$. This we can compute from $|x|_1$, $|y|_1$ and $a$ as follows. Let $z_x$ and $z_y$ be such that $x = z_x\tilde{x}$ and $y = z_y \tilde{x}$. We have
$$
\frac{|x|_1 + |y|_1 -a}{2} = \frac{|z_x|_1 + |z_y|_1 +2|\tilde{x}|_1 -a}{2}  = \frac{|z_x|_1 + a-|z_x|_1 +2|\tilde{x}|_1 -a}{2} = |\tilde{x}|_1.
$$
Thus, for any bit string $z\tilde{x}$ sampled by \simulateOnSubcube we can pass the \onemax value of $z$ to $A$.
In iteration $t$, when $A$ uses a $k$-ary unbiased operator which samples according to the distribution $D(\cdot \mid x^{(i_1)},\ldots,x^{(i_k)})$, \simulateOnSubcube  
uses the $(k+2)$-ary unbiased operator which samples according to the distribution $D'(\cdot \mid x^{(i_1)}\tilde{x},\ldots,x^{(i_k)}\tilde{x},x,y)$ such that
$$
\forall u \in \{0,1\}^n: D'(u \mid x^{(i_1)}\tilde{x},\ldots,x^{(i_k)}\tilde{x},x,y) = 
\begin{cases}
D(z \mid x^{(i_1)},\ldots,x^{(i_k)}), &\mbox{if }u = z\tilde{x};\\
0,	&\mbox{otherwise.}
\end{cases}
$$
For any $v^{(1)},\ldots,v^{(k+2)}$ such that $D'(\cdot \mid v^{(1)},\ldots,v^{(k+2)})$ is not defined by the equation just above, we let this distribution be the uniform distribution over $\{0,1\}^n$.
From the additional conditioning on $x$ and $y$ we see that $D'$ is indeed unbiased. Note that $D'$ samples only points from $[x,y]$. As described above, \simulateOnSubcube can now use the \onemax value of the resulting $z\tilde{x}$ to compute the \onemax value of $z$ and pass that on to $A$ as the answer to the query. This shows that the simulation is successful as desired.
\end{proof}

\begin{theorem}\label{thm:ternaryLongJump}
Let $\ell \leq (1/2 - \varepsilon)n$. For all $k \geq 3$, the $k$-ary unbiased black-box complexity of $\jump{\ell}$ is
 $O(\UBB{k-2}(\onemax))$.
\end{theorem}
\begin{proof}
We will optimize $\jump{\ell}$ blockwise, where each block is optimized by itself while making sure that the correct \onemax value is available as long as only bits within the block are modified. Afterwards, the different optimized blocks are merged to obtain the optimum. 

Let $a$ be such that $\ell = n/2 - a$ and assume for the moment that $a$ divides $n$.
Algorithm~\ref{alg:kAryJumpEll} gives a formal description of the intuitive idea sketched above. This algorithm uses the following unbiased operators.
\begin{itemize}
	\item $\uniformSample()$: The operator \uniformSample is a $0$-ary operator which samples a bit string uniformly at random.
	\item $\flipwhereequal_k(x,y)$: For two search points $x$ and $y$ and an integer $k$, the operator $\flipwhereequal_k$ generates a search point by randomly flipping $k$ bits in $x$ among those bits where $x$ and $y$ agree. If $x$ and $y$ agree in less than $k$ bits, then all bits where $x$ and $y$ agree are flipped.
	\item $\selectBit(x,y,z)$: For three search points $x,y$ and $z$, the operator \selectBit returns a bit string identical to the first argument, except where the second and third differ, there the bits of $x$ are flipped. Note that this operator is deterministic.
	\item $\copySecondIntoFirstWhereDifferentFromThird(x,y,z)$: For three search points $x,y$ and $z$, the operator \copySecondIntoFirstWhereDifferentFromThird copies $x$, except where the second differs from the third; there it copies $y$. This is also a deterministic operator.
\end{itemize}

Furthermore, we will used the subroutine \simulateOnSubcube from Lemma~\ref{lem:simulateOnSubcube} with a fixed time budget that guarantees constant success at each call, returning the best bit string found (note that, if $a$ does not divide $n$, the last call to \simulateOnSubcube has to be with respect to a different Hamming distance).
\begin{algorithm2e}
 \Repeat{$f(x) = n/2$}{$x \assign \uniformSample()$\;} 
 $z^{(0)} \assign x$\;
 \For{$i = 1$ \KwTo $n/a$}{
	$z^{(i)} \assign \flipwhereequal_a(z^{(i-1)},x)$\;
	$y^{(i)} \assign \selectBit(x,z^{(i-1)},z^{(i)})$\;
 }
 \For{$i = 1$ \KwTo $n/a$}{
	$u^{(i)} \assign \simulateOnSubcube(x,y^{(i)},a)$\;	}
	$b \assign x$\;
 \For{$i = 1$ \KwTo $n/a$}{
	$b \assign \copySecondIntoFirstWhereDifferentFromThird(b,u^{(i)},x)$\;
 }
\caption{$k$-ary unbiased black-box algorithm for $\jump{\ell}$.}
\label{alg:kAryJumpEll}
\end{algorithm2e} 

\paragraph{Expected number of queries.} A uniformly sampled bit string has exactly $n/2$ ones with probability $\Theta(1/ \sqrt{n})$, which shows that the first line takes an expected number of $\Theta(\sqrt{n})$ queries. Since $a \in \Theta(n)$, all loops have a constant number of iterations. The body of the second loop takes as long as a single optimization of \onemax with arity $k$, which is in $\Omega(n / \log n)$, so that initial sampling in line~1 makes no difference in the asymptotic run time. Thus, the total number of queries is $O(\UBB{k-2}(\onemax))$.

\paragraph{Correctness.}
The algorithm first generates a reference string $x$ with $f(x) = n/2$. The first loop generates bit strings $y^{(i)}$ which have a Hamming distance of $a$ to the reference string $x$; in this way the different $y^{(i)}$ partition the bit positions into $\lceil n/a \rceil$ sets of at most $a$ positions each. The next loop optimizes (copies of) $x$ on each of the selected sets of $a$ bits independently as if optimizing \onemax. For the bit strings encountered during this optimization we will always observe the correct \onemax value, as their Hamming distance to $x$ it as most $a$, and $x$ has exactly $n/2$ ones.
The last loop copies the optimized bit positions into $b$ by copying the bits in which $u^{(i)}$ and $x$ differ (those are the incorrect ones). This selects the correct bits in each segment with constant probability according to Lemma~\ref{lem:simulateOnSubcube}. As all segments have a constant independent failure probability, we get a constant overall failure probability (since $a$ is constant) and Remark~\ref{REM:HIGHPROB} concludes the proof. The proof trivially carries over to the case of $n$ not divisible by $a$.
\end{proof}

Thus, we immediately get the following corollary, using the known run time bounds for \onemax from~\cite{DoerrW12g}.

\begin{corollary}
Let $\ell \leq (1/2 - \varepsilon)n$. Then the unbiased black-box complexity of $\jump{\ell}$ is 
\begin{itemize}
\item $O(n \log n)$, for ternary variation operators;
\item $O(n / k)$, for $k$-ary variation operators with $4 \leq k \leq \log n$;
\item $\Theta(n / \log n)$, for $k$-ary variation operators with $k \geq \log n$ or unbounded arity.
\end{itemize}
\end{corollary}

Note the upper bound of $O(n \log n)$ for the ternary unbiased black-box-complexity which we will improve in Section~\ref{sec:extremeJump} to $O(n)$. For all higher arities, the theorem presented in this section gives the best known bound.

\subsection{Unary Unbiased Optimization of Long Jump Functions}
\label{sec:unaryLongJump2}

When optimizing a $\jump{\ell}$ function via unary unbiased operators, the only way to estimate the $\onemax$-value of a search point $x$ (equivalently, its Hamming distance $H(x,\textbf{1}_n)$ from the optimum), is by sampling suitable offspring that have a non-zero fitness. When $\ell$ is small, that is, many $\onemax$-values can be derived straight from the fitness, we can simply flip $\ell$ bits and hope that the retrieved fitness value is by $\ell$ smaller than $\onemax(x)$. 
This was the main idea for dealing with short jump functions (cf. Section~\ref{sec:shortJump}).

When $\ell$ is larger, this does not work anymore, simply because the chance that we only flip $1$-bits to zero is too small. Therefore, in this section, we resort to a sampling approach that, via strong concentration results, learns the expected fitness of the sampled offspring of $x$, and from this the $\onemax$-value of the parent $x$. This will lead to a unary unbiased black-box complexity of $O(n^2)$ for all jump functions $\jump{\ell}$ with $\ell \le n/2 - \eps n$.

\paragraph{Proof outline and methods.} Since we aim at an asymptotic statement, let us assume that $n$ is sufficiently large and even. Also, since we shall not elaborate on the influence of the constant $\eps > 0$, we may assume (by replacing $\eps$ by a minimally smaller value) that $\eps$ is such that $\eps n$ is even. 

A first idea to optimize $\jump{\ell}$ with $\ell = n/2 - \eps n / 2$ could be to flip each bit of the parent $x$ with probability $1/2 - \eps /2$. Such an offspring $u$ has an expected fitness of $n/2 - \eps n / 2 + \eps \onemax(x)$. If $\eps$ is constant, then by Chernoff bounds $O(n \log n)$ samples $u$ are enough to ensure that the average observed fitness $n/2 - \eps n/2 + \eps v$ satisfies $v = \onemax(x)$ with probability $1 - n^{-c}$, $c$ an arbitrary constant. This is enough to build a unary unbiased algorithm using $O(n^2 \log^2(n))$ fitness evaluations. 

We improve this first approach via two ideas. The more important one is to not flip an expected number of $n/2 - \eps n / 2$ bits independently, but to flip exactly that many bits (randomly chosen). By this, we avoid adding extra variation via the mutation operator. This pays off when $x$ already has many ones---if $\onemax(x) = n-a$, then we will observe that only $O(a \log n)$ samples suffice to estimate the $\onemax$-value of $x$ precisely (allowing a failure probability of $n^{-c}$ as before). 

The price for not flipping bits independently (but flipping a fixed number of bits) is that we have do deal with hypergeometric distributions, and when sampling repeatedly, with sums of these. The convenient way of handling such sums is to rewrite them as sums of negatively correlated random variables and then argue that Chernoff bounds also hold for these. This has been stated explicitly in~\cite{Doerr11bookchapter} for multiplicative Chernoff bounds, but not for additive ones. Since for our purposes an additive Chernoff bound is more convenient, we extract such a bound from the original paper~\cite{PanconesiS97}.

The second improvement stems from allowing a larger failure probability. This will occasionally lead to wrong estimates of $\onemax(x)$, and consequently to wrong decisions on whether to accept $x$ or not, but as long as this does not happen too often, we will still expect to make progress towards the optimum. To analyze this, we formulate the progress of the distance to the optimum as random walk and use the gambler's ruin theorem to show that the expected number of visits to each state is constant.

\paragraph{Estimating the distance to the optimum.}
We  start with some preliminary considerations that might be helpful for similar problems as well. Let $x \in \{0,1\}^n$. Let $a := a(x) := n - \onemax(x) = H(x,\mathbf{1}_n)$ be its Hamming distance from the all-ones string. Fix some enumeration $1, \ldots, a$ of the zero-bits of~$x$. Let $u$ be an offspring of $x$ obtained from flipping exactly $n/2 - \eps n / 2$ bits. For $i \in [a]$, define a $\{-1,+1\}$-valued random variable $X_i$ by $X_i = 1$ if and only if the $i$th zero of $x$ is flipped in $u$. Then $\onemax(u) = n/2 + \eps n /2 + \sum_{i \in [a]} X_i$ follows from the elementary fact that each flipping bit that was not zero in $x$ reduces the fitness by one, each other flipping bit increases it by one. 

By construction, $\Pr(X_i = 1) = (1 - \eps)/2$ and $\Pr(X_i = -1) = (1 + \eps)/2$. Consequently, $E(\onemax(u)) = n/2 + \eps n / 2 - \eps a$, which is in $[n/2 - \eps n / 2, n/2+\eps n /2]$ for all $a$.

Let $X = \sum_{i \in [a]} X_i$. Note that the $X_i$ are not independent. However, they are negatively correlated and for this reason still satisfy the usual Chernoff bounds. This was made precise in Theorem~1.16 and~1.17 of~\cite{Doerr11bookchapter}), however, only giving multiplicative Chernoff bounds (Theorem~1.9 in~\cite{Doerr11bookchapter}). Since for our purposes an additive Chernoff bound (sometimes called Hoeffding bound) is more suitable, we take a look in the original paper by Panconesi and Srivnivasan~\cite{PanconesiS97}. There, Theorem~3.2 applied with correlations parameter $\lambda = 1$ and the $\hat X_i$ simply being independent copies of the $X_i$ together with equation~(2) give the first part of the following lemma. By setting the random variables $Y_i := 1 - X_i$, the second claim of the lemma follows from the first one.
 
\begin{lemma}
  Let $X_1, \ldots, X_n$ be binary random variables. Let $X = \sum_{i \in [n]} X_i$.
\begin{enumerate}
	\item[(a)] Assume that for all $S \subseteq [n]$, we have $\Pr(\forall i \in S: X_i = 1) \le \prod_{i \in S} \Pr(X_i = 1)$. Then $\Pr(X \ge E(X) + d) \le \exp(-2d^2 / n)$.
	\item[(b)] Assume that for all $S \subseteq [n]$, we have $\Pr(\forall i \in S: X_i = 0) \le \prod_{i \in S} \Pr(X_i =0)$. Then $\Pr(X \le E(X) + d) \le \exp(-2d^2 / n)$.
\end{enumerate}
\end{lemma}

Note that our $\{-1,+1\}$-valued $X_i$ are derived from a hypergeometric distribution (which leads to random variables fulfilling the assumptions of both parts of the above lemma) via a simple affine transformation. Consequently,  the following corollary directly implied by the lemma above applies to these $X_i$.

\begin{corollary}
  Let $X_1, \ldots, X_n$ be $\{-1,+1\}$-valued random variables. Assume that for all $S \subseteq [n]$ and both $b \in \{-1,+1\}$, we have $\Pr(\forall i \in S: X_i = b) \le \prod_{i \in S} \Pr(X_i = b)$. Let $X = \sum_{i \in [n]} X_i$. Then $\Pr(|X - E(X)| \ge d) \le 2 \exp(-d^2/(2n))$.
\end{corollary}

From this, we observe that $\Pr(\onemax(u) \notin [n/2 - \eps n, n/2+\eps n]) \le \Pr(|\onemax(u) - E(\onemax(u))| > \eps n / 2) = 
\Pr(|X - E(X)| > \eps n / 2) \le 2 \exp(-\eps^2 n/8)$. In particular, 
\begin{equation}
\label{eq:visible}
  \Pr(f(u) = 0) \le 2 \exp(-\eps^2 n/8).
\end{equation}

Note that independent copies of sets of negatively correlated random variables again are negatively correlated. Let $Y$ be the sum of $T$ independent copies of $X$. Then the corollary again yields
\[\Pr(Y \notin [-(a+1/2)\eps T,-(a-1/2)\eps T]) = \Pr(|Y - E(Y)| > \eps T / 2) \le 2 \exp(- \eps^2 T / 8a).\]

Similarly, $\Pr(Y \notin [(3/2)E(X),E(Y)/2]) \le 2 \exp(-|E(Y)|^2/(8aT)) = 2 \exp(-a\eps^2 T/8)$. We summarize these findings in the following lemma.

\begin{lemma}\label{lem:sampling}
  Let $x \in \{0,1\}^n$. Let $u_1, \ldots, u_T$ be obtained independently from $x$, each by flipping exactly $n/2 - \eps n/2$ random bits. Let $s = \sum_{i \in [T]} \onemax(u_i)$ and $\hat a := -(s - T(n/2 - \eps n / 2)) / (T\eps)$. The probability that $\lfloor \hat a + 1/2\rfloor$ does not equal $a := a(x)$ is at most $2\exp(-\eps^2 T / 8 a)$. The probability that $\hat a$ is not in $[a/2,(3/2)a]$ is at most $2\exp(-\eps^2 T a / 8)$.
\end{lemma}

\begin{proof}
  By construction $y = s - T(n/2 - \eps n/2)$ that the same distribution as $Y$ above. The probability, that $y/(T\eps)$ deviates from its expectation $-a$, or equivalently, that $-y/(T\eps)$ deviates from its expectation $a$, by at least an additive term of $1/2$, is at most $2\exp(-\eps^2 T / 8 a)$. 
\end{proof}

Building on the previous analysis, we now easily derive an estimator for a $\onemax$-value not revealed by a jump function. It overcomes the possible problem of sampling an offspring with fitness zero by restarting the procedure using the command \breakandgotoone.

\begin{algorithm2e}[h]
  $s \assign 0$\;
  \For{$i = 1$ \KwTo $T$}{
    $u \assign \flipB_{n/2 - \eps n / 2}(x)$\;
    $f_u \assign f(u)$\;
    \lIf{$f_u = 0$}{\breakandgotoone\;}
    $s \assign s + (f_u - n/2 - \eps n / 2)$\;
    }
    \Return{$\lfloor -s/(T\eps) + 1/2\rfloor$}\;
\caption{$\estimate(x,T)$. Estimation of the Hamming distance of $x$ to the optimum via roughly $T$ samples to a jump functions (see Corollary~\ref{cor:estimate} for the details).}
\label{alg:estimate}
\end{algorithm2e}

\begin{corollary}\label{cor:estimate}
The function $\estimate$ takes as inputs a search point $x$ and an integer $T$, performs an expected number of at most $T / (1 - 2T\exp(-\eps^2 n / 8))$ fitness evaluations, if this number is positive, and returns an integer~$\hat a$. 

Assume that $\eps n \ge 5\sqrt{n \log n}$ and $T \in O(n)$. Then the expected number of fitness evaluations is $T + O(1/n^2)$. Let $a = a(x) = H(x,\mathbf{1}_n)$ denote the unknown Hamming distance of $x$ to the optimum. The probabilities for the events $\lfloor \hat a +1/2 \rfloor \neq a$ and $\hat a \notin [a/2,(3/2)a]$ are at most $2\exp(-\eps^2 T / 8a) + O(1/n^3)$ and $2\exp(-\eps^2 T a / 8) + O(1/n^3)$, respectively. If $T \ge 24a\ln(6n/a)/\eps^2$, these probabilities become $2/(6n/a)^3 + O(1/n^3)$ and $O(1/n^3)$.
\end{corollary}

\begin{proof}
A run of $\estimate$ in which the $\breakandgotoone$ statement is not executed uses exactly $T$ fitness evaluations. The probability that one execution of the for-loop leads to the execution of the $\breakandgotoone$ statement is at most $2T \exp(-\eps^2 n/ 8)$ by a simple union bound argument and (\ref{eq:visible}). If this number is less than one, then an expected total number of $(1 - 2T \exp(-\eps^2 n/ 8))^{-1}$ times the for-loop is started, given an expected total number of at most $T(1 - 2T \exp(-\eps^2 n/ 8))^{-1}$ fitness evaluations. 

  When $\eps \ge 5\sqrt{n \log n}$ and $T \in O(n)$, the probability for a restart is $2T \exp(-\eps^2 n/ 8) \in O(1/n^3)$, the expected number of fitness evaluations becomes $T + O(1/n^2)$. Consequently, conditioning on none of the $u_i$ in Lemma~\ref{lem:sampling} having a $\onemax$-value outside $[n/2 - \eps n, n/2+\eps n]$ changes the probabilities computed there by at most an additive $O(1/n^3)$ term. 
\end{proof}

\begin{algorithm2e}[Hh]
  \Repeat{$f(x) \neq 0$}{$x \assign \uniformSample()$\;} 
  $\alpha \assign f(x)$\;
  \While{$f(x) \neq n$}{
    $x' \assign \flipB_{1}(x)$\;
    \eIf{$f(x')=n$}{$x \assign x'$\;}{
      $\alpha' \assign g(x',\alpha)$\;
      $\alpha'' \assign g(x,\alpha)$\;
      \lIf{$\alpha' > \alpha''$}{$(x,\alpha) \assign (x',\alpha')$\;}
      }
    }	
\caption{Unary optimization of a jump function $f$ using a $p$-estimator $g$ (cf.~Definition~\ref{def:estimator}).}
\label{alg:unaryjump}
\end{algorithm2e}

The main argument of how such an estimator for the number of $0$s in a bit string can be used to derive a good black-box algorithm will be reused in a later section (in Theorem~\ref{thm:unaryExtremeJump}). Thus, we make the following definition.

\begin{definition}\label{def:estimator}
Let $f$ be a pseudo-Boolean function and let 
$p$ be a function that maps non-negative integers to non-negative integers. 
Let $g$ be an algorithm which takes as input a bit string $x$ and a natural number $A$ and uses $O(p(n) \alpha \log (2+n/\alpha))$ unary unbiased queries to $f$. We call $g$ a \emph{$p$-estimator using $f$} if, for all bit strings $x$, $a = n - \onemax(x)$, and for all $\alpha \in [a/2,3a/2]$ we have
\begin{itemize}
	\item $P(g(x,\alpha) \neq a) \leq \frac{a}{16n}$;
	\item $P(g(x,\alpha) \not\in [a/2,3a/2]) \in O(1/(p(n)n^3))$.
\end{itemize}
\end{definition}

\begin{lemma}\label{lem:usingApproximateOnemax}
Let $f$ be a pseudo-Boolean function such that, for some $p$, there is a $p$-estimator using $f$. Then the unary unbiased black-box complexity of $f$ is $O(p(n) n^2)$.
\end{lemma}
\begin{proof}
  Let us consider a run of Algorithm~\ref{alg:unaryjump}. 
  
 Let us first assume that all calls of the $p$ estimator $g$ return a value that lies in $[1/2,3/2]$ times the $a$-value of the first argument. 

  During a run of the algorithm, $a(x)$ performs a biased random walk on the state space $[0..n]$. The walk ends when the state $0$ is reached. From the definition of a $p$-estimator we derive the following bounds on the transition probabilities. From state $a$, with probability at least $(a/n)(1 - 2 (a/16n)) \ge (a/n)(7/8)$, we move to state $a-1$. This is the probability that a one-bit flip reduces the Hamming distance to the optimum times a lower bound on the probability that we correctly identify both the resulting $a$-value and the $a$-value of our current solution $x$. With probability at most $((n-a)/n)(a/16n + (a+1)/16n)$, we move from state $a$ to state $a+1$. If we neither move to $a-1$ or $a+1$, we stay in $a$. Observe that when conditioning on that we do not stay in $a$, then with probability at least $3/4$ we go to $a-1$ and with probability at most $1/4$, we go to $a+1$ (and these are coarse estimates).
  
  Our first goal is to show that the expected number of different visits to each particular state $a$ is constant. To this aim, we may regard a speedy version of the random walk ignoring transitions from $a$ to itself. In other words, we may condition on that in each step we actually move to a different state. For any state $i \in [0..n]$, let $e_i$ denote the expected number of visits to $a$ starting the walk from $i$. To be more precise, we count the number of times we leave state $a$ in the walk starting at $i$. We easily observe the following. For $i > a$, we have $e_i = e_a$, simply because we know that the walk at some time will reach $a$ (because $0$ is the only absorbing state). Hence we can split the walk started in $i$ into two parts, one from the start until the first visit to $a$ (this contributes zero to $e_i$) and the other from the first visit to $a$ until reaching the one absorbing state (this contributes $e_a$ to $e_i$). 
For $i < a$, we have $e_i = (1 - q_i) e_a$, where $q_i$ denotes the probability that a walk started in $i$ visits $0$ prior to $a$. The above implies that $e_{a-1} \le e_{a+1}$. Consequently, for the state $a$ itself, we may use the pessimistic estimates on the transition probabilities above and derive $e_a \le 1 + (1/4) e_{a+1} +  (3/4) e_{a-1} = 1 + (1/4 + (3/4)(1-q_{a-1})) e_a$. To prove that $e_a \in O(1)$, it thus suffices to show that $q_{a-1}$ is bounded from below by an absolute constant greater than zero. This follows easily from the gambler's ruin theorem (see, e.g.,~\cite[Theorem A.4]{Jansen13}). Consider a game in which the first player starts with $r = a-1$ dollar, the second with $s = 1$ dollars. In each round of the game, the first player wins with probability $P = 1/4$, the second with probability $Q = 3/4$. The winner of the round receives one dollar from the other player. The games ends when one player has no money left. In this game, the probability that the second player runs out of money before the first, is exactly 
  \[\frac{(Q/P)^r - 1}{(Q/P)^{r+s} - 1},\]
which in our game is $(3^{a-1} - 1) / (3^a-1)$.  Using the pessimistic estimates of the transition probabilities, we hence see that $1 - q_{a-1} \le (3^{a-1} - 1) / (3^a-1) \le 1/3$ as desired.

  We have just shown that the expected numbers of times the algorithm has to leave a state ``$a(x) = a$'' is constant. Since the probability of leaving this state in one iteration of the while-loop is at least $a/2n$, and by the definition of an estimator one iteration takes an expected number of $O(p(n) a \log (2+n/a))$ fitness evaluations, we see that the total number of fitness evaluations spent in state ``$f(x) = a$'' is $O(p(n) n \log (2+n/a))$, with all constants hidden in the $O$-notation being absolute constants independent of $a$ and $n$. Consequently, the expected total number of fitness evaluation in one run of the algorithm is at most $\sum_{a \in [n]} O(p(n) n \log (2+n/a)) = O(p(n)n^2)$. 
  
  So far we assumed that all $g(y,T)$ calls return a value that is in $[a(y)/2,(3/2)a(y)]$, and conditional on this, proved an expected optimization time of $O(p(n)n^2)$. By the definition of an estimator, the probability that we receive a value outside this interval is $O(1/p(n)n^3)$. Consequently, the probability that this happens within the first $O(p(n)n^2)$ fitness evaluations is at most $O(1/n)$. A simple Markov bound shows that Algorithm~\ref{alg:unaryjump} after $O(p(n)n^2)$ fitness evaluations (assuming the implicit constant high enough), with probability at least $1/2$ we found the optimum. Hence Remark~\ref{REM:HIGHPROB} proves the claim.
\end{proof}

\begin{theorem}\label{thm:unaryJumpLinearCorridor}
Let $\ell \leq (1/2 - \varepsilon)n$. The unary unbiased black-box complexity of $\jump{\ell}$ is
$O(n^2)$.
\end{theorem}
\begin{proof}
  By Corollary~\ref{cor:estimate},  $g(x,\alpha) := \estimate(x,\lceil 24 (2\alpha+1) \ln(6n/(2\alpha+1))/\eps^2 \rceil)$ is a $p$-estimator for $\jump{n/2 - \eps n}$ with $p(n)$ a sufficiently large constant. 
To see this, note that $a \mapsto a \ln(6n/a)$ is increasing for all $a \in [2n]$. 
  Consequently, Lemma~\ref{lem:usingApproximateOnemax} shows the claim.
\end{proof}

\section{Extreme Jump Functions}
\label{sec:extremeJump}

In this section, we regard the most extreme case of jump functions where all search points have fitness zero, except for the optimum and search points having exactly $n/2$ ones. Surprisingly, despite some additional difficulties, we still find polynomial-time black-box algorithms.

Throughout this section, let $n$ be even. We call a jump function $\jump{\ell}$ an \emph{extreme jump function} if $\ell = n/2 - 1$. Consequently, this functions is zero except for the optimum (where it has the value $n$) and for bit-string having $n/2$ ones (where it has the value $n/2$).

The information-theoretic argument of~\cite{DrosteJW06} immediately gives a lower bound of $\Omega(n)$ for the unbiased black-box complexities of extreme jump functions. The intuitive argument is that an unrestricted black-box algorithm needs to learn $n$ bits of information, but receives only a constant amount of information per query.

\begin{lemma}\label{lem:extremejumplowerbound}
For all arities $k$, the $k$-ary unbiased black-box complexity of an extreme jump function is $\Omega(n)$.
\end{lemma}

\begin{proof} 
  Since an extreme jump function takes only three values, Theorem~2 in~\cite{DrosteJW06}  gives a lower bound of $\Omega(n)$ for the unrestricted black-box complexity of the set of all extreme jump functions. The latter is  a lower bound for the unbiased black-box complexity of a single extreme jump function (cf.~the end of Section~\ref{sec:model}). 
\end{proof}

\subsection{The Upper Bounds on Extreme Jump Functions}
\label{sec:upperextreme}

In the following three subsections, we shall derive several upper bounds for the black-box complexities of extreme jump functions. 

Notice that, for an extreme jump function, we cannot distinguish between having a \onemax value of $n/2+k$ and $n/2-k$ until we have encountered the optimum or its inverse. More precisely, let $x^{(1)}, x^{(2)}, ...$ be a finite sequence of search points not containing the all-ones and all-zeroes string. Define $y^{(i)} = x^{(i)} \oplus \textbf{1}_n$ to be the inverse of $x^{(i)}$ for all $i$. Then both these sequences of search points yield exactly the same fitness values. Hence the only way we could find out on which side of the symmetry point $n/2$ we are would be by querying a search point having no or $n$ ones. However, if we know such a search point, we are done anyway. 

Despite these difficulties, we will develop a linear time ternary unbiased black-box algorithm in the following section. In Section~\ref{sec:twoaryextreme}, we show that restricting ourselves to binary variation operators at most increases the black-box complexity to $O(n \log n)$. For unary operators, the good news derived in the final subsection of this section is that polynomial-time optimization of extreme jump functions is still possible, though the best complexity we find is only $O(n^{9/2})$.

To ease the language, let us denote by $d(x) := |\onemax(x) - n/2|$ a \emph{symmetricized version of $\onemax$} taking into account this difficulty. Also, let us define the \emph{sign} $\sgn(x)$ of $x$ to be $-1$, if $\onemax(x) < n/2$, $\sgn(x) := 0$, if $\onemax(x) = n/2$, and $\sgn(x) = +1$, if $\onemax(x) > n/2$. In other words, $\sgn(x)$ is the sign of $\onemax(x) - n/2$.

\subsection{Ternary Unbiased Optimization of Extreme Jump Functions}

When ternary operators are allowed, we quite easily obtain an unbiased black-box complexity of $O(n)$, which is best possible by Lemma~\ref{lem:extremejumplowerbound}. The reason for this fast optimization progress is that, as explained next, we may test individual bits. Assume that we have a search point $u$ with $\onemax$-value $n/2+1$. If we flip a certain bit in $u$, then from the fitness of this offspring, we learn the value of this bit. If the new fitness is $n/2$, then the $\onemax$-value is $n/2$ as well and the bit originally had the value one. If the new fitness is zero, then the new $\onemax$-value is $n/2+2$ and the original bit was set to one. We thus can learn all bit values and flip those bits which do not have the desired value.

One difficulty to overcome, as sketched in Section~\ref{sec:upperextreme}, is that we will never have a search point where we know that its $\onemax$-value is $n/2+1$. We overcome this by generating a search point with fitness $n/2$ and flipping a single bit. This yields a search point with $\onemax$-value either $n/2+1$ or $n/2-1$. Implementing the above strategy in a sufficiently symmetric way, we end up with a search point having $\onemax$-value either $n$ or $0$ and in the latter case output its complement. 

\begin{theorem}
For $k \ge 3$, the $k$-ary unbiased black-box complexity of extreme jump functions is $\Theta(n)$.
\end{theorem}

\begin{proof}
We show that Algorithm~\ref{alg:ternaryJumpNDivTwo} optimizes any extreme jump function with $O(n)$ black-box queries using only operators of arity at most 3. This algorithm uses the three operators \uniformSample, \selectBit and \flipwhereequal introduced in the proof of Theorem~\ref{thm:ternaryLongJump}, as well as the following unbiased operator.
\begin{itemize}
	\item $\complement(x)$: Given a bit string $x$, \complement flips all bits in $x$. This is a deterministic operator.
\end{itemize}

\begin{algorithm2e}
 \Repeat{$f(x) \neq 0$}{$x \assign \uniformSample()$\;} 
 $z^{(0)} \assign x$\; 
 \For{$i = 1$ \KwTo $n$}{
	$z^{(i)} \assign \flipwhereequal_1(z^{(i-1)},x)$\;
	$y^{(i)} \assign \selectBit(x,z^{(i-1)},z^{(i)})$\;
 }
 $b \assign y^{(1)}$\;
 \For{$i = 2$ \KwTo $n$}{
	$a^{(i)} \assign \selectBit(x,y^{(1)},y^{(i)})$\;
	\If{$f(a^{(i)}) = 0$}{$b \assign \selectBit(b,x,y^{(i)})$\;}
 }
Sample $\complement(b)$\;
\caption{Ternary unbiased black-box algorithm for an extreme jump function $f$.}
\label{alg:ternaryJumpNDivTwo}
\end{algorithm2e} 

Note that a uniformly sampled bit string has exactly $n/2$ ones with probability $\Theta(1/ \sqrt{n})$. Consequently, the expected total number of queries is at most $4n + O(\sqrt n) = \Theta(n)$.

Let us now analyze the correctness of our algorithm. Let $x$ be the initial bit string with fitness different from $0$. This is either the optimal bit string, in which case nothing is left to be done, or a bit string with fitness $n/2$. In the latter case, consider the first for-loop. For each $i \leq n$, $z^{(i)}$ has a Hamming distance of $i$ to $x$; in fact, the sequence $(z^{(i)})_{0 \leq i \leq n}$ is a path in the hypercube flipping each bit exactly once. Thus, for each $i$, $y^{(i)}$ differs from $x$ in exactly one position, and for each position there is exactly one $i$ such that $y^{(i)}$ differs from $x$ in that position. We can thus use the $y^{(i)}$ to address the individual bits, and we will call the bit where $x$ and $y^{(i)}$ differ \emph{the $i$th bit}. 

We use $y^{(1)}$ now as a ``base line'' and check which other bits in $x$ contribute to the \onemax value of $x$ in the same way (both are $0$ or both are $1$) as follows. The bit string $a^{(i)} = \selectBit(x,y^{(1)},y^{(i)})$ is obtained from $x$ by flipping the first and $i$th bit. Thus, the fitness of $a^{(i)}$ is $0$ if and only if the first and the $i$th bit contribute to the \onemax value of $x$ in the same way, otherwise it is $n/2$.

Thus, $b$ is the bit string with either all bits set to $0$ or all bits set to $1$. This means that we are either done after the last loop or after taking the complement of $b$.
\end{proof}

\subsection{Binary Unbiased Optimization of the Extreme Jump Function}\label{sec:twoaryextreme}

In this section, we prove that the unbiased $2$-ary black-box complexity of extreme jump functions is $O(n \log n)$. With $2$-ary operators only, it seems impossible to implement the strategy used in the previous subsection, which relies on being able to copy particular bit values into the best-so-far solution.

To overcome this difficulty, we follow a hill-climbing approach. We first find a search point $m$ with $d$-value $0$ by repeated sampling. We copy this into our ``current-best'' search point $x$ and try to improve $x$ to a new search point $x'$ by flipping a random bit in which $x$ and $m$ are equal  (this needs a $2$-ary operation), hoping to gain a search point with $d$-value equal to $d(x) + 1$. The main difficulty is to estimate the $d$-value of $x'$, which is necessary to decide whether we keep this solution as new current-best or whether we try again.

  Using binary operators, we can exploit the fact that $H(x,m) = d(x)$. For example, we can flip $d(x)-1$ of the $d(x)+1$ bits in which $x'$ and $m$ differ. 
  If this yields an individual with fitness $n/2$, then clearly $x'$ has not the targeted $d$-value of $d(x)+1$. Unfortunately, we detect this shortcoming only when the bit that marks the difference of $x$ and $x'$ is not among the $d(x)-1$ bits flipped. This happens only with probability $2 / (d(x)+1)$. Consequently, this approach may take $\Theta(n)$ iterations to decide between the cases $d(x') = d(x) +1$ and $d(x') = d(x) -1$.

  We can reduce this time to logarithmic using the following trick. 
  Recall that the main reason for the slow decision procedure above is that the probability of not flipping the newly created bit is so small. 
  This is due to the fact that the only way to gain information about $x'$ is by flipping almost all bits so as to possibly reach a fitness of $n/2$. 
  We overcome this difficulty by in parallel keeping a second search point $y$ that has the same $d$-value as $x$, but is ``on the other side'' of $m$. To ease the language in this overview, let us assume that $\onemax(x) > n/2$. Let $k := d(x)$ and $H(m,x) = k$. Then we aim at keeping a $y$ such that $d(y) = k$, $H(m,y) = k$, $H(x,y) = 2k$, and $\onemax(y) = n/2-k$. With this at hand, we can easily evaluate the $d$-value of $x'$. Assume that $x'$ was created by flipping exactly one of the bits in which $x$ and $y$ agree. Let $u$ be created by flipping in $x'$ exactly $k-1$ of the bits in which $x'$ and $y$ differ. 
  If $d(x') = k+1$, then surely $d(u) = 2$, and thus $f(u) = 0$. 
  If $d(x') = k-1$, then with probability $(k+2)/(2k+1) \ge 1/2$ the bit in which $x'$ and $x$ differ is not flipped, leading to $\onemax(u) = n/2$, visible from a fitness equal to $n/2$. Hence, with probability at least $1/2$, we detect the undesired outcome $d(x') = k-1$. Unfortunately, there is no comparably simple certificate for ``$d(x') = k+1$'', so we have to repeat the previous test $2 \log n$ times to be sufficiently sure (in the case no failure is detected) that $d(x') = k+1$. Overall, this leads to an almost linear complexity of $O(n \log n)$.  

To make this idea precise, let us call a pair $(x,y)$ of search points \emph{opposing} if they have opposite signs, that is, if $\sgn(x) \sgn(y) = -1$. We call $(x,y)$ an opposing $k$-pair for some integer $k$ if $x$ and $y$ are opposing, $d(x) = d(y) = k$, and $H(x,y) = 2k$. Clearly, in this case one of $x$ and $y$ has a $\onemax$-value of $n/2-k$, while the other has one of $n/2+k$. To further ease the language, let us call bits having the value one \emph{good}, those having value zero \emph{bad}. Then the definition of an opposing $k$-pair implies that in one of $x$ or $y$ all the bit-positions $x$ and $y$ differ in are good, whereas in the other, they are all bad. The remaining positions contain the same number of good and bad bits. 

There are some additional technicalities to overcome. For example, since we cannot decide the sign of a search point, it is non-trivial to generate a first opposing pair.  Our solution is to generate different $x$ and $y$ in Hamming distance one from $m$ and create offspring of $x$ and $y$ via a mixing crossover $\mix$ that inherits exactly one of the two bits $x$ and $y$ differ in from $x$, the other from $y$. If $x$ and $y$ are opposing, then this offspring with probability one has a fitness of $n/2$. Otherwise, is has a fitness of $n/2$ only with probability $1/2$. Consequently, a polylogarithmic number of such tests with sufficiently high probability distinguishes the two cases.

Before giving the precise algorithms, let us define the operators used. We use the operator $\flipwhereequal_k$ introduced in the proof of Theorem~\ref{thm:ternaryLongJump} as well as the following operators.
\begin{itemize}
	\item $\flipwheredifferent_k(x,y)$: For two search points $x$ and $y$ and an integer $k$, the operator $\flipwheredifferent_k$ generates a search point by randomly flipping $k$ bits in $x$ among those bits where $x$ and $y$ disagree. If $x$ and $y$ disagree in less than $k$ bits, a random bit-string is returned.
	\item $\mix(x,y)$: If $x$ and $y$ disagree in exactly two bits, then a bit-string is returned that inherits exactly one of these bits from $x$ and one from $y$ and that is equal to both $x$ and $y$ in all other bit-positions. If $x$ and $y$ do not disagree in exactly two bits, a random bit-string is returned.
\end{itemize}

We are now in a position to formally state the algorithms. We start with the key routine $\movefirst_k$, which, from an opposing $k$-pair $(x,y)$, computes a Hamming neighbor $x'$ of $x$ with $d(x') = d(x) + 1$. Applying this function to both $(x,y)$ and $(y,x)$, we will obtain an opposing $(k+1)$-pair in the main algorithm (Algorithm~\ref{alg:TwoAryJump}).

\begin{algorithm2e}
status $\assign$ failure\;
\While{\normalfont status $\neq$ success}{
  $x' \assign \flipwhereequal_{1}(x,y)$\;
  status $\assign$ success\;
  $i \assign 0$\;
  \While{\normalfont (status = success) and ($i \le 2 \log n$)}{
    $i \assign i + 1$\;
    $u \assign \flipwheredifferent_{k-1}(x',y)$\;
    \lIf{$f(u) = n/2$}{status $\assign$ failure\;}
  }
}
\Return $x'$\;     
\caption{subroutine $\movefirst_k$, {$k \in [1, \ldots, n/2-1]$}, which takes $x,y \in \{0,1\}^n$ as input and returns an $x' \in \{0,1\}^n$. If $(x,y)$ is an opposing $k$-pair, then the output $x'$ with probability at least $1 - (2 \log n)/n^2$ satisfies $d(x') = k + 1$ and $H(x',y) = 2k + 1$.}  
\end{algorithm2e}

\begin{lemma}\label{lem:movefirst}
  Let $k \in [1..n/2 - 1]$. The function $\movefirst_k$ is binary unbiased. Assume that it is called with an opposing $k$-pair $(x, y)$. Let $X$ be a binomially distributed random variable with success probability $1/2$ and denote by $T$ the random variable counting the number of fitness evaluations in one run of $\movefirst_k$. Then the following holds. 
\begin{enumerate}
\item[(i)] $T$ is dominated by $(1 + 2\log n) X$, also if we condition on the output satisfying $d(x') = k+1$. 
\item[(ii)] With probability at least $1 - (2\log n)/n^2$, the output $x' \in \{0,1\}^n$ satisfies $d(x') = k+1$, $H(x,x')=1$, and $H(x',y) = 2k + 1$.
\end{enumerate}
\end{lemma}

\begin{proof}
  Since the operators $\flipwhereequal$ and $\flipwheredifferent$ are $2$-ary and unbiased, $\movefirst_k$ is a $2$-ary unbiased algorithm. Also note that any $x'$ generated in line~3 necessarily has $H(x,x') =1$ and $H(x',y) = 2k+1$, the latter because $H(x,y) = 2k$ and $x'$ is obtained from $x$ by flipping a bit $x$ and $y$ agree in. Hence the main challenge is to distinguish between the two cases that $d(x') = k+1$ and $d(x') = k-1$. 
  
  We first argue that the inner while-loop terminates surely with ``status = success'', if $d(x') = k+1$, and with probability $1 - 1/n^2$ terminates with ``status = failure'' if $d(x') = k-1$. Assume first that $d(x') = k+1$. Since $x'$ and $y$ then are opposing and $d(y) = k$, the $2k+1$ bits $x'$ and $y$ differ in are all good bits in $x'$ (and thus bad bits in $y$) or vice versa. Consequently, flipping any $k-1$ of them in $x'$ surely reduces its $d$-value to $2$, leading to $f(u) = 0$. Hence ``status = success'' is never changed to ``status = failure''. Assume now that $d(x') = k-1$. Then $x'$ and $y$ differ in $2k$ good bits and one bad bit, or in $2k$ bad bits and one good bit. Therefore, with probability $(k+2)/(2k+1) \ge 1/2$, all $k-1$ bits flipped in the creation of $u$ are of the same type. In this case, the $\onemax$-values of $u$ and $x'$ differ by $k-1$, implying $f(u) = n/2$. Thus, a single iteration of the while-loop sets the status variable to failure with probability at least $1/2$, as desired when $d(x') = k-1$. The probability that this does not happen in one of the up to $2 \log n$ iterations is at most $1/n^2$. 
  
  We now analyze the $\flipwhereequal$ statement in line 3. If $(x,y)$ is an opposing $k$-pair, then exactly half of the bits $x$ and $y$ agree in are good  and the other half are bad. In either case, flipping one of the agreeing bits in $x$ has a chance of exactly $1/2$ of increasing the $d$-value (and the same chance of $1/2$ of decreasing it). By what we proved about the inner while-loop above, we see that the random variable describing the number of iterations of the outer while-loop is dominated by a binomial random variable with success probability $1/2$ (it would be equal to such a binomial random variable if the inner while-loop would not with small probability accept a failure as success). Since each execution of the inner while-loop leads to at most $2 \log n$ fitness evaluations, this proves the first part of (i). If we condition on the output satisfying $d(x') = k+1$, then indeed the inner while-loop does not misclassify an $x'$. Consequently, the number of iterations in the outer while-loop has distribution $X$, and again $T$ is dominated by $(1 + 2\log n)X$.
  
  For the failure probability estimate, we use the following blunt estimate. Since the $\flipwhereequal$ statement with probability $1/2$ produces an $x'$ that we view as success (and that will become the output of the function finally), with probability at least $1 - 1/n^2$ there will be at most $2 \log n$ iterations of the outer while-loop each generating a failure-$x'$. Each of them has a chance of at most $1/n^2$ of being misclassified as success. Consequently, the probability that $\movefirst_k$ returns a failure-$x'$ is at most $1 - (2\log n)/n^2$.
\end{proof}

\begin{algorithm2e}[t]
 \Repeat{$f(m) = n/2$}{$m \assign \uniformSample()$\;} 
 status $\assign$ failure\;
 \While{\normalfont status = failure}{
  $x \assign \flipwhereequal_1(m,m)$\;
  $y \assign \flipwhereequal_1(m,x)$\;
  status $\assign$ success\;
  $i \assign 0$\;
  \While{\normalfont (status $\neq$ failure) and ($i \le \sqrt n$)}{
    $i \assign i+1$\;
    $u \assign \mix(x,y)$\;
    if $f(u) \neq n/2$ then status $\assign$ failure\;
    }
  }
 \For{$k = 1$ \KwTo $n/2-1$}{
  $x' \assign \movefirst_k(x,y)$\;
  $y' \assign \movefirst_k(y,x)$\;
  $(x,y) \assign (x',y')$\;
  }
  \eIf{$f(x) = n$}{\Return $x$\;}{\Return $y$\;} 
\caption{A $2$-ary unbiased black-box algorithm that for any extreme jump function $f$ with high probability finds the optimum in $O(n \log n)$ iterations.}
\label{alg:TwoAryJump}
\end{algorithm2e} 

\begin{theorem}\label{thm:TwoAryJump}
  Algorithm~\ref{alg:TwoAryJump} is a $2$-ary unbiased black-box algorithm for extreme jump functions. It finds the optimum of an unknown extreme jump function with probability $1 - o(1)$ within $O(n \log n)$ fitness evaluations.
\end{theorem}

\begin{proof}
  As in the previous section, line~1 of Algorithm~\ref{alg:TwoAryJump} found a search point $m$ having fitness $n/2$ after $O(n)$ fitness evaluations with probability $1 - \exp(\Omega(\sqrt n))$.
  
  Lines~2 to~11 are devoted to generating an opposing $1$-pair $(x,y)$. Since $m$ has exactly $n/2$ good and bad bits, $x$ has a $\onemax$-values of $n/2 - 1$ or $n/2+1$, each with probability exactly $1/2$. Independent of this outcome, $y$ in line~5 has a chance of $(n/2) / (n-1) > 1/2$ of having the opposite value, which means that $(x,y)$ is an opposing $1$-pair. 
  
  Observe that if $x$ and $y$ are opposing, then $\mix(x,y)$ with probability one has a fitness of $n/2$, simply because both possible outcomes of $\mix(x,y)$ have this fitness. If $x$ and $y$ are not opposing, then  $x$ and $y$ have both one good bit more than $m$ or both have one more bad bit. Consequently, the one outcome of $\mix(x,y)$ that is different from $m$ has a $d$-value of~$2$, visible from a fitness different from $n/2$. We conclude that the inner while-loop using at most $\sqrt n$ fitness evaluations surely ends with ``status = success'', if $(x,y)$ is an opposing $1$-pair, and with probability $1 - 2^{-n}$ ends with ``status = failure'', if not. A coarse estimate thus shows that after $O(n)$ fitness evaluations spent in lines~2 to~11, which involves at least $\Omega(\sqrt n)$ executions of the outer while-loop, with probability $1 - \exp(\Omega(\sqrt n))$ we exit the outer while-loop with an opposing $1$-pair $(x,y)$.
  
  We now argue that if $(x,y)$ is an opposing $k$-pair right before line~13 is executed, then with probability at least $1 - (4 \log n)/n^2$, the new $(x,y)$ created in line~15 is an opposing $(k+1)$-pair. This follows from twice applying Lemma~\ref{lem:movefirst} and noting that the condition ``$H(x,x') = 1$'' in the statement of Lemma~\ref{lem:movefirst} implies that $H(x,y) = 2(k+1)$ after executing line~15. Consequently, a simple induction shows that with probability at least $1 - O(\log(n) / n)$, the pair $(x,y)$ when leaving the for-loop is an opposing $1$-pair, and thus, exactly one of $x$ and $y$ is the optimum.
  
  For the run time statement, let us again assume that all opposing $k$-pairs are indeed created as in the previous paragraphs. We already argued that up to before line~12, with probability $1 - \exp(\Omega(\sqrt n))$, we spent at most $O(n)$ fitness evaluations. We then perform $n-2$ calls to $\movefirst_k$ functions, each leading to a number of fitness evaluations which is, independent from what happened in the other calls, dominated by $(1+2\log n)$ times a binomial random variable with success probability $1/2$. By Lemma~1.20 of~\cite{Doerr11bookchapter}, we may assume in the following run time estimate that the number of fitness evaluations in each such call is indeed $(1 + 2\log n)$ times such an (independent) binomial random variable. By Theorem~1.14 of~\cite{Doerr11bookchapter}, the probability that a sum of $\Theta(n)$ such binomial random variables deviates from its $\Theta(n)$ expectation by more than a factor of two is $\exp(-\Theta(n))$. Consequently, with high probability the total number of fitness evaluations is $O(n \log n)$.
\end{proof}

Note that by using different constants in the $\movefirst_k$ algorithm we can easily improve the success probability from the stated $1 - o(1)$ to any $1 - O(n^{-c})$, $c$ a constant. We do not care that much about this, because black-box complexity usually deals with expected number of fitness evaluations. To obtain now such a result, our claim in Theorem~\ref{thm:TwoAryJump} is sufficient.

\begin{corollary}
  There is a $2$-ary unbiased black-box algorithm solving the ``extreme jump functions'' problem in an expected number of $O(n \log n)$ queries.
\end{corollary}

\begin{proof}
	By Theorem~\ref{thm:TwoAryJump}, a single run of this budgeted version of Algorithm~\ref{alg:TwoAryJump} with probability $1 - o(1)$ finds the optimum. Thus, with Remark~\ref{rem:highprobability}, we can conclude that we have an overall expected number of black-box queries.
\end{proof}

\subsection{Unary Unbiased Optimization of Extreme Jump Functions}
\label{sec:1extreme}

With the next theorem we show that, surprisingly, even the \emph{unary} unbiased black-box complexity of extreme jump functions is still polynomially bounded. Note that now we cannot learn the $\onemax$-value of a search point $x$ by repeatedly flipping a certain number of bits and observing the average objective value. Since $n/2$ is the only non-trivial objective value, any such average will necessarily be $n/2$ (except in the unlikely event that we encountered the optimum). The solution is to flip, depending on parity reasonings, exactly $n/2 -1$ or $n/2$ bits in $x$ and note that the probability $p_a$ of receiving a search point with (visible) fitness $n/2$ depends on the distance $a:=a(x) = \min\{|x|_1,|x|_0:=n-|x|_1\}$ of $x$ to optimum or its opposite 
We roughly have $p_a \in \Theta(a^{-1/2})$ and $p_{a-2} - p_{a} \in \Theta(a^{-3/2} n^{-1} (n - 2a))$. These small numbers lead to the fact that for $a \in n/2 - \Theta(1)$, we will need $\Theta(n^{9/2})$ samples to estimate the $a$-value of $x$ with constant probability. Since estimating becomes easier for smaller $a$, we are able to construct a unary unbiased black-box algorithm finding the optimum of an extreme jump function in an expected total of $O(n^{9/2})$ fitness evaluations.

\begin{theorem}\label{thm:unaryExtremeJump}
Let $n$ be even and $\ell = n/2 -1$. Then the unary unbiased black-box complexity of $\jump{\ell}$ is $O(n^{9/2})$.
\end{theorem}

We start by analyzing the probability of receiving a bit-string $y$ with fitness $n/2$ when flipping a certain number $\lambda$ of bits in a fixed bit-string $x$. Note that when $\lambda$ has a parity different from the one of $\onemax(x)$, our offspring $y$ will never have a fitness of $n/2$. Consequently, we need to treat the cases of $\onemax(x)$ even or odd separately. Note that since $n$ is even, the parities of $\onemax(x)$ and $a(x)$ are equal.

In the following, we will frequently need the following estimate due to \cite{Robbins55}.
\begin{equation}\label{eq:centralBinomialCoefficients}
\forall n \ge 1: \frac{4^n}{\sqrt{2\pi n}} \leq \binom{2n}{n} \leq \frac{4^n}{\sqrt{\pi n}}.
\end{equation}

Let us first \textbf{assume that $a := a(x)$ is even}. Let $y$ be obtained from $x$ by flipping exactly $n/2$ random bits. Note that $\Pr(f(y) - n/2)$ depends on $a$, but not more specifically on $x$ or $\onemax(x)$. We may thus write $p_a := \Pr(f(y) = n/2)$. Note that $f(y) = n/2$ holds if and only if we flipped half the $1$-bits and half the $0$-bits in $x$. Hence, 
\begin{equation}\label{eq:paeven}
p_a = \frac{\binom{n-a}{(n-a)/2}\binom{a}{a/2}}{\binom{n}{n/2}},
\end{equation}
which for $a \ge 1$ implies
\begin{equation}\label{eq:paevenest}
p_a \in \Theta\left(\frac{1}{\sqrt{a}}\right)
\end{equation}
by equation~(\ref{eq:centralBinomialCoefficients}). 

Let $a \ge 2$. We are interested in the difference $p_{a-2} - p_{a}$ in order to estimate the number of samples that are necessary to discriminate between the cases $a(x) = a$ and $a(x) = a-2$. By equation~(\ref{eq:paeven}) we have
\begin{align}\label{eqn:diff}
& p_{a-2} - p_{a} \nonumber \\
 & = \frac{\binom{n-a+2}{(n-a+2)/2}\binom{a-2}{(a-2)/2}}{\binom{n}{n/2}} - \frac{\binom{n-a}{(n-a)/2}\binom{a}{a/2}}{\binom{n}{n/2}}\nonumber\\
& = \frac{\binom{n-a+2}{(n-a+2)/2} \binom{a}{a/2}}{\binom{n}{n/2}}\left[\frac{\binom{a-2}{(a-2)/2}}{\binom{a}{a/2}} - \frac{\binom{n-a-2}{(n-a-2)/2}}{\binom{n-a+2}{(n-a+2)/2}}\right]\nonumber\\
& = \frac{\binom{n-a+2}{(n-a+2)/2} \binom{a}{a/2}}{\binom{n}{n/2}}\left[\frac{(a-2)!}{((a-2)/2)!^2}\frac{(a/2)!^2}{a!} - \frac{(n-a)!}{((n-a)/2)!^2}\frac{((n-a+2)/2)!^2}{(n-a+2)!}\right]\nonumber\\
& = \frac{\binom{n-a+2}{(n-a+2)/2} \binom{a}{a/2}}{\binom{n}{n/2}}\left[\frac{(a/2)^2}{a(a-1)} - \frac{((n-a+2)/2)^2}{(n-a+2)(n-a+1)}\right]\nonumber\\
& = \frac{\binom{n-a+2}{(n-a+2)/2} \binom{a}{a/2}}{\binom{n}{n/2}} \frac{1}{4} \left[\frac{a}{a-1} - \frac{n-a+2}{n-a+1}\right]\nonumber\\
& = \frac{\binom{n-a+2}{(n-a+2)/2} \binom{a}{a/2}}{\binom{n}{n/2}} \frac{1}{4} \left[\frac{(n-a+1)a - (n-a+2)(a-1)}{(n-a+1)(a-1)}\right]\nonumber\\
& = \frac{\binom{n-a+2}{(n-a+2)/2} \binom{a}{a/2}}{\binom{n}{n/2}} \frac{1}{4} \left[\frac{n-2a+2}{(n-a+1)(a-1)}\right].
\end{align}

\begin{lemma}\label{lem:samplethemiddleeven}
  For any constant $k$, there is a constant $K$ such that the following holds. Let $a$ be even. Let $y \in \{0,1\}^n$ such that its distance from the optimum or its opposite $a(y) := \min\{|y|_1, |y|_0\}$ equals $a$. Consider the following random experiment. Let $N \ge  N_a := K a^{5/2} n^2 / (n-2a)^{3/2}$. Exactly $N$ times we independently sample an offspring of $y$ by flipping exactly $n/2$ bits. Let $Y$ denote the number of times we observed an offspring with fitness exactly $n/2$. Then, if $a \ge 2$, \[\Pr(Y \ge N (p_a + p_{a-2})/2) \le \exp(-k(n-2a)^{1/2}).\]
  If $a \le n/2 - 2$, then \[\Pr(Y \le N (p_{a} + p_{a+2})/2) \le \exp(-k(n-2a)^{1/2}).\]
\end{lemma}

\begin{proof}
  To prove the first part, observe that by definition $Y$ is the sum of $N \ge N_a$ independent indicator random variables which are one with probability $p_a$. Consequently, $Y \ge N (p_a + p_{a-2})/2$ is equivalent to 
  \[Y \ge E(Y) (1 + \tfrac{p_{a-2} - p_{a}}{2p_a}).\] 
Writing $\delta = (p_{a-2} - p_{a})/(2p_a)$, we compute from equations~(\ref{eq:paevenest}) and~(\ref{eqn:diff}) that $\delta \ge c (n - 2a) n^{-1} a^{-1}$ for some absolute constant $c$. By a standard Chernoff bound, $\Pr(Y \ge  N (p_a + p_{a-2})/2) \le \exp(-\delta^2 E(Y) / 2) \in \exp(-\Theta(K) (n-2a)^{1/2})$. By choosing $K$ large enough, the first claim follows. The second claim can be proven with analogous arguments.
\end{proof}

\textbf{For $a = a(x)$ odd}, things are not much different, so we only sketch some details. If $a(x)$ is odd, $x$ has an odd number both of ones and zeroes, and these numbers are $a$ and $n-a$ in some order. If $\lfloor |x|_1 / 2 \rfloor$ of the $1$-bits and $\lfloor |x|_0 /2 \rfloor $ of the $0$-bits flip, we obtain a bit-string with $|x|_1 - \lfloor |x|_1 / 2 \rfloor + \lfloor |x|_0 /2 \rfloor = n/2$ ones, that is, with fitness $n/2$. Also, there is no other way of obtaining a fitness of $n/2$ by flipping $n/2 - 1 = \lfloor |x|_0 / 2 \rfloor + \lfloor |x|_1 / 2 \rfloor$ bits. Consequently, the probability that flipping exactly $n/2 - 1$ bits in $x$ gives a bit-string $y$ with fitness $n/2$ is, independent of the particular $x$, 
\begin{equation}\label{eq:paodd}
p_a := \frac{\binom{n-a}{(n-a-1)/2} \binom{a}{(a-1)/2}}{\binom{n}{n/2 - 1}} \in \Theta\left(\frac 1 {\sqrt a}\right),
\end{equation}
where the asymptotic statement follows, among others, from noting that $\binom{2k+1}{k} = (1/2) \binom{2(k+1)}{k+1}$ and $\binom{2k}{k-1} = \frac{k}{k+1}\binom{2k}{k}$ valid for all $k$ and (\ref{eq:centralBinomialCoefficients}).

Analogous to the case of even $a$, we compute for all odd $a \ge 3$ that \[p_{a-2} - p_a = \frac 14 \frac{\binom{n-a+2}{(n-a+1)/2}\binom{a}{(a-1)/2}}{\binom{n}{n/2 - 1}} \frac{n - 2a + 2}{a(n-a+2)}.\]
From this, we derive the following analogue of Lemma~\ref{lem:samplethemiddleeven}.

\begin{lemma}\label{lem:samplethemiddleodd}
  For any constant $k$, there is a constant $K$ such that the following holds. Let $a$ be odd. Let $y \in \{0,1\}^n$ such that its distance from the optimum or its opposite $a(y) := \min\{|y|_1, |y|_0\}$ equals $a$. Consider the following random experiment. Let $N \ge  N_a := K a^{5/2} n^2 / (n-2a)^{3/2}$. Exactly $N$ times we independently sample an offspring of $y$ by flipping exactly $n/2-1$ bits. Let $Y$ denote the number of times we observed an offspring with fitness exactly $n/2$. Then, if $a \ge 3$, \[\Pr(Y \ge N (p_a + p_{a-2})/2) \le \exp(-k(n-2a)^{1/2}).\]
  If $a \le n/2 - 2$, then \[\Pr(Y \le N (p_{a+2} + p_{a})/2) \le \exp(-k(n-2a)^{1/2}).\]
\end{lemma}

The above observations immediately yield an estimator for an invisible $a$-value of a bit-string $y$. Since we shall only need it to distinguish between two possible values $a(y) = a-1$ and $a(y) = a+1$, we formulate the estimator and the following lemma tailored for this purpose. It is clear that stronger statements, relying less on a pre-knowledge on $a(y)$, could be derived as well. 

\begin{algorithm2e}[t]
  \If{$a = n/2 -1$}
    {\eIf{$f(y) = n/2$}{\Return $n/2$\;}{\Return $n/2-2$\;}}
  \If{$a = 1$}
    {\eIf{\normalfont $f(y) = n$ or $f(\flipB_{n}(y))=n$}{\Return $n$\;}{\Return $n-2$\;}}
  $N \assign \max\{N_{a-1},N_{a+1}\}$\;
  $Y \assign 0$\;
  \For{$i = 1$ \KwTo $N$}{
    \eIf{\normalfont $a$ odd}{$z \assign \flipB_{n/2}(x)$\;}{$z \assign \flipB_{n/2 -1}(x)$\;}
    \If{$f(z) = n/2$}{$Y \assign Y+1$\;}
    }
  \eIf{$Y < N(p_{a+1}+p_{a-1})/2$}{\Return($a+1$)\;}{\Return($a-1$)\;}  
\caption{The procedure $\estimate(y,a)$ analyzed in Lemma~\ref{lem:estimate}.}
\label{alg:estimateextreme}
\end{algorithm2e} 

\begin{lemma}\label{lem:estimate}
  The procedure $\estimate$ described in Algorithm~\ref{alg:estimateextreme} has the following properties. On arbitrary inputs $y \in \{0,1\}^n$ and $a \in [1\,..\,n/2-1]$, it performs $\Theta(a^{5/2} n^2 (n-2a)^{-3/2})$ fitness evaluations. If $a(y) \in \{a-1,a+1\}$, then with probability at least $1 - \exp(-\Theta((n-2a)^{1/2}))$ the true value of $a(y)$ is returned. By choosing the implicit constant in the first statement sufficiently large, this success probability can be made arbitrary close to one.  
\end{lemma}

\begin{proof}
  The claim is trivial for $a = 1$ and $a = n/2 -1$, hence let $a \in [2\,..\,n/2-2]$. Assume first that $a$ is odd, consequently, $a(y)$ is even. If $a(y) = a+1$, then by Lemma~\ref{lem:samplethemiddleeven}, with probability at least $1 - \exp(-\Theta((n-2a)^{1/2}))$ we have $Y < N(p_{a+1} + p_{a-1})/2$ leading to the correct return value of $a+1$. If $a(y) = a-1$, the second part of the lemma yields that we have $Y > N(p_{a+1} + p_{a-1})/2$ with probability at least $1 - \exp(-\Theta((n-2a)^{1/2}))$, which again leads to the correct return value of $a-1$. The case of even $a$ is done analogously by invoking Lemma~\ref{lem:samplethemiddleodd}. 
\end{proof}

\SetKw{KwDownTo}{DownTo}
\begin{algorithm2e}
 \Repeat{$f(m) = n/2$}{$m \assign \uniformSample()$\;} 
 $x \assign \flipB_1(m)$\;
 \For{$a = n/2-1$ \KwDownTo $1$}{
   status $\assign$ failure\;
   \While{\normalfont status = failure}{
     $y \assign \flipB_1(x)$\;
     \If{$\estimate(y,a) = a-1$}{
       $x \assign y$\;
       status $\assign$ success\;
       }
     }
   }
Sample $\complement(x)$\;
\caption{A unary unbiased black-box algorithm that for any extreme jump function $f$ with constant probability finds the optimum in $O(n^{9/2})$ fitness evaluations. }
\label{alg:UnaryExtreme}
\end{algorithm2e}

\begin{theorem}
  The unary unbiased black-box complexity of extreme jump functions is $O(n^{9/2})$. This is witnessed by Algorithm~\ref{alg:UnaryExtreme}, which with constant probability finds the optimum of an extreme jump function using $O(n^{9/2})$ fitness evaluations.
\end{theorem}

\begin{proof}
  The first sentence of the theorem follows from the second and Remark~\ref{rem:highprobability}.
  
  In the analysis of Algorithm~\ref{alg:UnaryExtreme}, let us first assume that all $\estimate(y,a)$ calls with $a(y) \in \{a-1,a+1\}$ return $a(y)$ correctly. 
  
  Assume that we start the while-loop in Algorithm~\ref{alg:UnaryExtreme} with an $x$ such that $a$ is equal to $a(x)$. Since $y$ is a Hamming neighbor of $x$, we have $a(y) \in \{a-1,a+1\}$. If $a(y) = a+1$, nothing changes. If $a(y) = a-1$, this is again correctly detected in the if-clause, and the while-loop is left with $x \assign y$. Consequently, we start the following iteration of the for-loop again with $a(x) = a$. The probability of generating a $y$ with $a(y) = a-1$ is $a/n$. Consequently, the while-loop is left after an expected number of $n/a$ iterations. 
  
  The expected total number of fitness evaluations in the for-loop is now easily computed as 
  \[\sum_{a = 1}^{n/2-1} (n/a) O(a^{5/2} n^2 / (n-2a)^{3/2}) \in O(n^{9/2}) \sum_{a = 1}^{n/2-1}(n-2a)^{-3/2} = O(n^{9/2}).\]
  
  The above is true if we assume that none of the exceptional events (``failures'') of Lemma~\ref{lem:estimate} occurs. We now argue that in fact with constant probability, none of them occurs. To  this aim, we estimate the expected number of first failures in a typical run of the algorithm (a first failure is one where all previous calls of the $\estimate$ procedure did not fail). Consider one iteration of the while loop. If $a \neq a(x)$, then a failure must have occurred before, hence the probability now for a first failure is zero. If $a = a(x)$, we can invoke Lemma~\ref{lem:estimate} and deduce that this iteration has a failure probability of at most $\exp(-k(n-2a)^{1/2})$, where $k$ is a sufficiently large absolute constant.
  
  We may further assume, for the sake of this argument, that a failure is immediately corrected by some external authority. Note that this only changes the run of the algorithm after the occurrence of the first failure. So in particular, it does not change the expected number of first failures. By this, however, we may assume that the expected number of iterations done with $x$ having a certain $a$-value is exactly $n/a$. Consequently, the expected number of first failures is at most $\sum_{a = 1}^{n/2-1} (n/a) \exp(-k(n-2a)^{1/2}) = \sum_{b = 1}^{n/2-1} n (n/2 -b)^{-1} \exp(-k\sqrt b) \in O(1)$, where the implicit constant can be made arbitrarily small by the appropriate choice of $k$. Hence, with constant probability there is no first failure, and thus, also no failure at all.
\end{proof}

\section{Summary and Outlook}
We have analyzed the unbiased black-box complexity of short, long, and extreme jump functions. 
Along the way, we have introduced new tools for such analyses. 
Our work raises a number of interesting questions for future research. 

Since our focus was on deriving new ideas for the design of new search heuristics, we did not undertake in this work a complete investigation of all possible combinations of arity and jump size, but rather highlighted prominent complexity behaviors and prototypical algorithmic ideas. Still, it would be interesting to have a more complete picture than Table~\ref{tab:results}, in particular, making clear how far certain algorithmic ideas take and where certain regimes change. 

Another interesting line of research would be results that are more precise than just the asymptotic order. For example, it seems reasonable that for $\ell$ small enough, the unary unbiased black-box complexity of $\jump{\ell}$ is not only of the same order as the one of $\onemax$, but equal apart from lower order terms (which might actually be surprisingly small). Note that such precise analyses for run times of given algorithms recently attracted quite some interest, see~\cite{BottcherDN10,Witt13j,Sudholt13,DoerrD14,HwangPRTC14} and the references therein.

Furthermore, we are optimistic that some of the algorithmic ideas developed in the previous sections can be used to design new search heuristics.

\subsection*{Acknowledgments}

Parts of this work have been done during the Dagstuhl
seminar 10361 ``Theory of Evolutionary Algorithms''.

}

\end{document}